\documentclass{article}
\usepackage{times}
\usepackage{xcolor}
\usepackage{soul}
\usepackage{amsmath,amsthm,amssymb,amsfonts,bbm}
\usepackage{algorithm}
\usepackage[noend]{algpseudocode}
\usepackage[utf8]{inputenc}
\usepackage[small]{caption}
\usepackage{subfigure}
\usepackage[pdftex]{graphicx} 
\usepackage{todonotes}

\usepackage[letterpaper]{geometry}
\usepackage[parfill]{parskip}
\PassOptionsToPackage{numbers, compress}{natbib}
\usepackage[numbers, compress]{natbib}
\usepackage{xr-hyper,refcount}

\usepackage{graphicx}
\usepackage[utf8]{inputenc} 
\usepackage[T1]{fontenc}    
\usepackage[colorlinks]{hyperref}       
\usepackage{url}            
\usepackage{booktabs}       
\usepackage{amsfonts}       
\usepackage{nicefrac}       
\usepackage{microtype}      
\usepackage[parfill]{parskip}
\usepackage{mathtools}
\usepackage{cases}
\usepackage{comment}
\usepackage{color}
\usepackage{appendix}
\usepackage{lmodern}
\usepackage[lining,semibold]{libertine}
\usepackage[T1]{fontenc}
\usepackage[libertine]{newtxmath}
\usepackage{bm}
\usepackage{xspace}
\usepackage{enumitem}
\usepackage[english]{babel}
\usepackage{authblk}

\usepackage{cleveref}
\crefformat{equation}{(#2#1#3)}
\crefrangeformat{equation}{(#3#1#4) to~(#5#2#6)}
\crefname{equation}{}{}
\Crefname{equation}{}{}

\crefname{definition}{\textbf{definition}}{definitions}
\Crefname{definition}{Definition}{Definitions}
\crefname{assumption}{\textbf{assumption}}{assumptions}
\Crefname{assumption}{Assumption}{Assumptions}
\definecolor{maroon}{RGB}{192,80,77}

	\AtEndDocument{\refstepcounter{theorem}\label{finalthm}}

\usepackage{mdwlist}

\newcommand{\RRR}{\mathcal{R}}
\newtheorem{theorem}{Theorem}
\theoremstyle{definition}
\newtheorem{definition}{Definition}[section]
\newcommand{\namecite}[1]{\citeauthor{#1}~[\citeyear{#1}]}

\title{Combating Reinforcement Learning's Sisyphean Curse \\ 
with Intrinsic Fear}

\begin{document}
\author{
  Zachary C. Lipton$^{1,2,3}$, Kamyar Azizzadenesheli$^4$, Abhishek Kumar$^3$, Lihong Li$^{5}$, \quad 
Jianfeng Gao$^6$, Li Deng$^7$\\
  Carnegie Mellon University$^1$,
  Amazon AI$^2$,
  University of California, San Diego$^3$,
Univerisity of California, Irvine$^4$,
Google$^5$,
Microsoft Research$^6$,
Citadel$^7$
\\
  \href{mailto:zlipton@cmu.edu}{\nolinkurl{zlipton@cmu.edu}},
   \href{mailto:kazizzad@uci.edu}{\nolinkurl{kazizzad@uci.edu}},
    \href{mailto:abkumar@ucsd.edu}{\nolinkurl{abkumar@ucsd.edu}}\\
    \{
    \href{mailto:lihongli.cs@gmail.com}{\nolinkurl{lihongli}},
   \href{mailto:jfgao@microsoft.com}{\nolinkurl{jfgao}},
    \href{mailto:l.deng@ieee.org}{\nolinkurl{deng}}
    \}
    @microsoft.com
}

\maketitle

\begin{abstract}
Many practical environments contain 
catastrophic states that an optimal agent would visit infrequently or never.
Even on toy problems, 
Deep Reinforcement Learning (DRL) agents 
tend to periodically revisit
these states upon forgetting their existence
under a new policy.
We introduce \emph{intrinsic fear} (IF), 
a learned reward shaping 
that guards DRL agents against periodic catastrophes.
IF agents possess a \emph{fear} model 
trained to predict the probability 
of imminent catastrophe. 
This score is then used to penalize the 
Q-learning objective.
Our theoretical analysis 
bounds the reduction in average return
due to learning on the perturbed objective. 
We also prove robustness to classification errors.
As a bonus, IF models tend to learn faster, 
owing to reward shaping.
Experiments demonstrate that \emph{intrinsic-fear} DQNs 
solve otherwise pathological environments 
and improve on several Atari games.

\end{abstract}

\section{Introduction}
\label{sec:introduction}
Following the success of deep reinforcement learning (DRL)
on Atari games \cite{mnih15human} 
and the board game of Go \cite{silver2016mastering}, 
researchers are increasingly exploring 
practical applications.
Some investigated applications include robotics \cite{levine2016end}, 
dialogue systems \cite{fatemi2016policy,lipton2016efficient},
energy management \cite{googleenergy2016}, 
and self-driving cars \cite{shalev2016safe}. 
Amid this push to apply DRL,
we might ask, 
\emph{can we trust these agents in the wild?}
Agents acting society may cause harm.
A self-driving car might hit pedestrians
and a domestic robot might injure a child.
Agents might also cause self-injury,
and while Atari lives lost are inconsequential, 
robots are expensive.

Unfortunately, it may not be feasible 
to prevent all catastrophes
without requiring extensive prior knowledge 
\cite{garcia2015comprehensive}.
Moreover, for typical DQNs, 
providing large negative rewards does not solve the problem: as soon as the catastrophic trajectories are flushed from the replay buffer, the updated Q-function ceases to discourage revisiting these states.


\textbf{In this paper}, we define
\emph{avoidable catastrophes}
as states that prior knowledge dictates 
an optimal policy should visit rarely or never.
Additionally, we define \emph{danger states}---those from which a catastrophic state 
can be reached in a small number of steps,
and assume that the optimal policy
does visit the danger states rarely or never.
The notion of a danger state
might seem odd absent any assumptions 
about the transition function.
With a fully-connected transition matrix,
all states are danger states. 
However, physical environments are not fully connected. 
A car cannot be parked this second, 
underwater one second later.

This work primarily addresses 
how we might prevent DRL agents
from perpetually making the same mistakes.
As a bonus, we show that the prior knowledge knowledge 
that \emph{catastrophic states} should be avoided
accelerates learning.
Our experiments show that even on simple toy problems, 
the classic deep Q-network (DQN) algorithm
fails badly,
repeatedly visiting catastrophic states
so long as they continue to learn. 
This poses a formidable obstacle to using DQNs in the real world. 
How can we trust a DRL-based agent
that was doomed to periodically 
experience catastrophes, 
just to remember that they exist?
Imagine a self-driving car 
that had to periodically hit a few pedestrians 
to remember that it is undesirable.

In the tabular setting, an RL agent never forgets the learned dynamics of its environment, 
even as its policy evolves.
Moreover, when the Markovian assumption holds, 
convergence to a globally optimal policy is guaranteed. 
However, the tabular approach becomes infeasible in high-dimensional, continuous state spaces.
The trouble for DQNs owes to the use of function approximation \cite{murata2005memory}.
When training a DQN,
we successively update a neural network based on experiences.
These experiences might be sampled in an online fashion, 
from a trailing window 
(\emph{experience replay buffer}),
or uniformly from all past experiences.
Regardless of which mode we use to train the network,
eventually, states that a learned policy never encounters 
will come to form an infinitesimally small region of the training distribution. 
At such times, our networks suffer the well-known problem of catastrophic forgetting \cite{mccloskey1989catastrophic,mcclelland1995there}.
Nothing prevents the DQN's policy from drifting back towards
one that revisits forgotten catastrophic mistakes. 

We illustrate the brittleness of modern DRL algorithms
with a simple pathological problem
called \emph{Adventure Seeker}. 
This problem consists of a one-dimensional continuous state, two actions,
simple dynamics, and admits an analytic solution.
Nevertheless, the DQN fails.
We then show that similar dynamics exist in the classic RL environment Cart-Pole. 

To combat these problems,
we propose the \emph{intrinsic fear} (IF) algorithm. 
In this approach,
we train a supervised \emph{fear model}
that predicts which states are likely to lead 
to a catastrophe within $k_r$  steps.
The output of the fear model (a probability),
scaled by a \emph{fear factor} 
penalizes the $Q$-learning target.
Crucially, the fear model maintains buffers 
of both \emph{safe} and \emph{danger} states. 
This model never forgets danger states,
which is possible due to the infrequency of catastrophes.

We validate the approach both empirically and theoretically.
Our experiments address \emph{Adventure Seeker}, Cartpole, and several Atari games.
In these environments, we label every lost \emph{life} 
as a catastrophe.
On the toy environments, 
IF agents learns to avoid catastrophe indefinitely.    
In Seaquest experiments, the IF agent achieves higher reward
and in Asteroids, the IF agent achieves both higher reward and fewer catastrophes.
The improvement on Freeway is most dramatic.

We also make the following theoretical contributions: 
First, we prove that when the reward is bounded 
and the optimal policy rarely visits the danger states, 
an optimal policy learned on the perturbed reward function 
has approximately the same return as the optimal policy learned on the original value function. 
Second, we prove that our method is robust to noise in the danger model.

\section{Intrinsic fear}
\label{sec:intrinsic-fear}
An agent interacts with its environment via 
a Markov decision process, or MDP, 
$(\mathcal{S}, \mathcal{A}, \mathcal{T}, \mathcal{R}, \gamma)$.
At each step $t$, the agent observes a state $s\in \mathcal{S}$
and then chooses an action $a\in \mathcal{A}$ according to its policy $\pi$.
The environment then transitions to state $s_{t+1} \in \mathcal{S}$ according to transition dynamics $\mathcal{T}(s_{t+1}|s_t,a_t)$ and generates a reward $r_t$ with expectation $\mathcal{R}(s,a)$.
This cycle continues until each episode terminates. 

An agent seeks to maximize 
the cumulative discounted return 
$\sum_{t=0}^T \gamma^t r_{t}$.
Temporal-differences methods \cite{sutton1988learning} 
like Q-learning \cite{watkins1992qlearning} 
model the Q-function, which gives the \emph{optimal} discounted total reward of a state-action pair.
Problems of practical interest tend to have large state spaces, 
thus the Q-function is typically approximated 
by parametric models such as neural networks.  
 
In Q-learning with function approximation, 
an agent collects experiences by acting greedily with respect to $Q(s,a;\theta_Q)$ and updates its parameters $\theta_Q$.
Updates proceed as follows.
For a given experience $(s_t,a_t,r_t,s_{t+1})$,
we minimize the squared Bellman error: 
\begin{equation}
\mathcal{L} = (Q(s_t,a_t; \theta_Q)-y_t)^2
\end{equation}
for $y_t = r_t + \gamma \cdot \max_{a'} Q(s_{t+1}, a'; \theta_Q)$.
Traditionally, the parameterised $Q(s,a;\theta)$ is trained by stochastic approximation, 
estimating the loss on each experience as it is encountered, yielding the update:
\begin{equation}
\begin{split}
\theta_{t+1} \gets & \theta_t + \alpha (
y_t - Q(s_t,a_t;\theta_t)) \nabla Q(s_t,a_t;\theta_t)\,.
\end{split}
\end{equation}
Q-learning methods also require an exploration strategy for action selection.  For simplicity, we consider only the $\epsilon$-greedy heuristic.
A few tricks 
help to stabilize Q-learning 
with function approximation.
Notably, with experience replay \cite{lin1992self}, 
the RL agent maintains a buffer of experiences, 
of experience to update the Q-function.

We propose a new formulation: 
Suppose there exists a subset $\mathcal{C} \subset \mathcal{S}$
of known \emph{catastrophe states}/
And assume that for a given environment, 
the optimal policy rarely enters 
from which catastrophe states are reachable 
in a short number of steps.
We define the distance $d(s_i, s_j)$
to be length $N$ of the smallest sequence of transitions
$\{(s_t, a_t, r_t, s_{t+1})\}_{t=1}^{N}$
that traverses state space from $s_i$ to $s_j$.\footnote{In the stochastic dynamics setting, the distance is the minimum mean passing time between the states.}
\begin{definition}
Suppose a priori knowledge 
that acting according to the optimal policy $\pi^*$,
an agent rarely encounters states $s\in S$ 
that lie within distance $d(s, c) < k_\tau$ 
for any catastrophe state $c \in \mathcal{C}$. 
Then each state $s$ for which $\exists c \in \mathcal{C} 
\text{ s.t. }  d(s,c) < k_\tau$ is a \emph{danger state}.
\end{definition}

\begin{algorithm}
\caption{Training DQN with Intrinsic Fear}\label{euclid}
\begin{algorithmic}[1]
\State \textbf{Input: } $Q$ (DQN), $F$ (fear model), fear factor $\lambda$, fear phase-in length $k_{\lambda}$, fear radius $k_r$
\State \textbf{Output: } Learned parameters $\theta_Q$ and $\theta_F$
\State Initialize parameters $\theta_Q$ and $\theta_F$ randomly
\State Initialize replay buffer $\mathcal{D}$, danger state buffer $\mathcal{D}_D$, and safe state buffer $\mathcal{D}_S$
\State Start per-episode turn counter $n_e$
\For{$t$ in 1:$T$}
	\State With probability $\epsilon$ select random action $a_t$
    \State Otherwise, select a greedy action $a_t = \arg\max_{a} Q(s_{t}, a; \theta_{Q})$
  	\State Execute action $a_t$ in environment, observing reward $r_t$ and successor state $s_{t+1}$
	\State Store transition $(s_t, a_t, r_t, s_{t+1})$ in  $\mathcal{D}$
    \If {$s_{t+1}$  is a catastrophe state}
 	\State Add states $s_{t-k_r}$ through $s_{t}$ to $\mathcal{D}_D$
    \Else
	     \State Add states $s_{t-n_e}$ through $s_{t-k_r-1}$ to $\mathcal{D}_S$ 
	\EndIf
  \State Sample a random mini-batch of transitions ($s_{\tau}$ , $a_{\tau}$, $r_{\tau}$, $s_{\tau+1}$) from $\mathcal{D}$
 \State $\lambda_\tau \gets \min(\lambda, \frac{\lambda \cdot t}{k_{\lambda}})$
 \State $y_{\tau} \gets \left\{\begin{array}{ll}
 \text{for terminal  } s_{\tau+1}:&\\
     \qquad  r_{\tau} - \lambda_\tau \\
     \text{for non-terminal } s_{\tau+1}:&\\
\qquad       r_\tau + \max_{a'} Q(s_{\tau+1}, a'; \theta_{Q}) - \\
\qquad\qquad\qquad\qquad\quad\lambda \cdot F(s_{\tau+1} ; \theta_{F}) \\
\end{array}\right\} $ 
\State $\theta_Q \gets \theta_Q - \eta \cdot \nabla_{\theta_Q} (y_\tau - Q(s_{\tau}, a_{\tau}; \theta_{Q}))^2$ 
\State Sample random mini-batch $s_j$ with $50\%$ of examples from $\mathcal{D}_D$ and $50\%$ from $\mathcal{D}_S$ 
\State $y_{j} \gets  \left\{\begin{array}{ll}
    	1, & \text{for } s_j \in \mathcal{D}_D \\
        0, & \text{for } s_j \in \mathcal{D}_S \\
       \end{array}\right\} $
\State $\theta_F \gets \theta_F - \eta \cdot \nabla_{\theta_F} \text{loss}_F(y_j, F(s_j; \theta_F)) $
\EndFor
\end{algorithmic}
\label{alg:intrinsic-fear}
\end{algorithm}


In 
Algorithm~\ref{alg:intrinsic-fear}, 
the agent maintains both a DQN 
and a separate, supervised \emph{fear model} 
$F:\mathcal{S}\mapsto[0,1]$. 
$F$ provides an auxiliary source of reward,
penalizing the Q-learner for entering likely danger states.
In our case, we use a neural network 
of the same architecture as the DQN 
(but for the output layer).
While one could sharing weights between the two networks, 
such tricks are not relevant to this paper's contribution.

We train the fear model to predict  
the probability that any state will lead 
to catastrophe within $k$ moves.
Over the course of training,
our agent adds each experience
$(s,a,r,s')$ to its experience replay buffer.
Whenever a catastrophe is reached
at, say, the $n_{th}$ turn of an episode,
we add the preceding $k_r$ (\emph{fear radius}) states 
to a \emph{danger buffer}.
We add the first $n-k_r$ states of that episode
to a \emph{safe buffer}.
When $n<k_r$, all states for that episode 
are added to the list of danger states.
Then after each turn, 
in addition to updating the Q-network,
we update the fear model,
sampling $50\%$ of states from the \emph{danger buffer}, assigning them label $1$,
and the remaining $50\%$ from the \emph{safe buffer}, 
assigning them label $0$. 

For each update to the DQN,
we perturb the TD target $y_t$.
Instead of updating $Q(s_t,a_t;\theta_{Q})$ towards $r_t + \max_{a'} Q(s_{t+1}, a'; \theta_{Q})$, we modify the target by subtracting the \emph{intrinsic fear}:
\begin{equation}
y_t^{IF} = r_t + \max_{a'} Q(s_{t+1}, a'; \theta_{Q}) - \lambda \cdot F(s_{t+1} ; \theta_{F})
\end{equation}
where $F(s; \theta_{F})$ is the fear model and $\lambda$ is a \emph{fear factor} determining the scale of the impact of intrinsic fear on the Q-function update.




\section{Analysis}

Note that IF perturbs the objective function. 
Thus, one might be concerned 
that the perturbed reward 
might lead to a sub-optimal policy. 
Fortunately, as we will show formally, if the labeled catastrophe states 
and danger zone do not violate our assumptions,
and if the fear model reaches arbitrarily high accuracy, 
then this will not happen. 

For an MDP, $M=\langle\mathcal{S}, \mathcal{A}, \mathcal{T}, \mathcal{R}, \gamma\rangle$, with $0\leq\gamma\leq 1$, the average reward return is as follows:
%
%

\begin{align*}
\eta_M(\pi) :=
\begin{cases}
\lim_{T\rightarrow \infty}\frac{1}{T}\mathbb{E}_M\Big[\sum_t^T r_t|\pi\Big]~~~&\textit{if}~~~\gamma = 1 \\[7pt]
(1-\gamma)\mathbb{E}_M\Big[\sum_t ^{\infty}\gamma^tr_t|\pi\Big] &\textit{if}~~~ 0\leq\gamma<1
\end{cases}
\end{align*}

The optimal policy $\pi^*$ of the model $M$ is the policy which maximizes 
the average reward return, 
$\pi^* = \max_{\pi\in\mathcal{P}} \eta(\pi)$ 
where $\mathcal{P}$ is a set of stationary polices.
\begin{theorem} \label{thm:ideal}
For a given MDP, $M$, 
with $\gamma\in[0,1]$ 
and a catastrophe detector $f$, 
let $\pi^*$ denote \emph{any} optimal policy of $M$, and $\tilde{\pi}$ denote 
an optimal policy of $M$ 
equipped with fear model $F$, and $\lambda$, environment $(M,F)$. 
If the probability $\pi^*$ 
visits the states in the danger zone 
is at most $\epsilon$, and $0\leq\RRR(s,a)\leq 1$, then
\begin{align}
\eta_M(\pi^*)\geq\eta_{M}(\tilde{\pi}) \geq \eta_{M,F}(\tilde{\pi})\geq \eta_M(\pi^*)-\lambda\epsilon\,.
\end{align}
In other words, $\tilde{\pi}$ is $\lambda\epsilon$-optimal in the original MDP.
\end{theorem}
\begin{proof}
The policy $\pi^*$ visits the fear zone with probability at most $\epsilon$. Therefore, applying $\pi^*$ on the environment with intrinsic fear $(M,F)$, provides a expected return of at least $\eta_M(\pi^*)-\epsilon\lambda$. Since there exists a policy with this expected return on $(M,F)$, therefore, the optimal policy of $(M,F)$, must result in an expected return of at least $\eta_M(\pi^*)-\epsilon\lambda$ on $(M,F)$, i.e. $\eta_{M,F}(\tilde{\pi} )\geq \eta_M(\pi^*)-\epsilon\lambda$. The expected return $\eta_{M,F}(\tilde{\pi} )$ decomposes into two parts: $(i)$ the expected return from original environment $M$, 
$\eta_M(\tilde{\pi})$, $(ii)$ the expected return from the fear model. If $\tilde{\pi}$ visits the fear zone with probability at most  $\tilde{\epsilon}$, then $\eta_{M,F}(\tilde{\pi})\geq \eta_M(\tilde{\pi}) - \lambda\tilde{\epsilon}$. 
Therefore, applying $\tilde{\pi}$ on $M$ promises an expected return of at least $\eta_M(\pi^*)-\epsilon\lambda +\tilde{\epsilon}\lambda$, lower bounded by $\eta_M(\pi^*)-\epsilon \lambda$.
\end{proof}

It is worth noting that the theorem holds for \emph{any} optimal policy of $M$.  If one of them does not visit the fear zone at all (i.e., $\epsilon = 0$), then $\eta_M(\pi^*)= \eta_{M,F}(\tilde{\pi}) $ and the fear signal can boost up the process of learning the optimal policy.

Since we empirically learn the fear model $F$ using collected data of some finite sample size $N$, 
our RL agent has access 
to an imperfect fear model $\hat{F}$, and therefore, computes the optimal policy based on $\hat{F}$. 
In this case, the RL agent trains with intrinsic fear generated by $\hat{F}$, 
learning a different value function 
than the RL agent with perfect $F$. 
To show the robustness against errors in $\hat{F}$, we are interested 
in the average deviation 
in the value functions 
of the two agents.

Our second main theoretical result, given in Theorem~\ref{thm:classifier}, allows the RL agent to use a smaller discount factor, denoted $\gamma_{plan}$, than the actual one ($\gamma_{plan}\leq \gamma$), to reduce the planning horizon and computation cost.
Moreover, when an estimated model of the environment is used, \namecite{jiang2015dependence} shows that using a smaller discount factor for planning may prevent over-fitting to the estimated model.  Our result demonstrates that using a smaller discount factor for planning can reduce reduction of expected return when an estimated fear model is used.


Specifically, for a given environment, 
with fear model $F_1$ and discount factor $\gamma_{1}$, 
let $V_{F_1,\gamma_1}^{\pi^*_{F_2,\gamma_{2}}}(s),~s\in\mathcal{S},$ 
denote the state value function 
under the optimal policy of an environment with fear model $F_2$ and the discount factor $\gamma_2$. 
In the same environment, 
let $\omega^{\pi}(s)$ denote the visitation distribution over states under policy $\pi$.
We are interested in the average reduction on expected return caused by an imperfect classifier; this reduction, denoted $\mathcal{L}(F,\widehat{F},\gamma,\gamma_{plan})$, is defined as
\begin{align*}
(1-\gamma)\int_{s\in\mathcal{S}}\omega^{\pi^*_{\widehat{F},\gamma_{plan}}}(s)
\left(V_{F,\gamma}^{\pi^*_{F,\gamma}}(s)-V_{F,\gamma}^{\pi^*_{\widehat{F},\gamma_{plan}}}(s)\right)ds\,.
\end{align*}
 
\begin{theorem}\label{thm:classifier}
Suppose $\gamma_{plan} \le \gamma$, and $\delta\in(0,1)$. Let $\hat{F}$ be the fear model in $\mathcal{F}$ with minimum empirical risk on $N$ samples.
For a given MDP model, the average reduction on expected return, $\mathcal{L}(F,\widehat{F},\gamma,\gamma_{plan})$, vanishes as $N$ increase:
with probability at least $1-\delta$,
\begin{align}\label{eq:classification}
\mathcal{L}=\mathcal{O}\left(\lambda\frac{1-\gamma}{1-\gamma_{plan}}\frac{\mathcal{VC}(\mathcal{F})+\log{\frac{1}{\delta}}}{N} + \frac{(\gamma-\gamma_{plan})}{1-\gamma_{plan}}\right)\,,
\end{align}
where $\mathcal{VC}(\mathcal{F})$ is the $\mathcal{VC}$ dimension of the hypothesis class $\mathcal{F}$.
\end{theorem}
\begin{proof}
In order to analyze $\left(V_{F,\gamma}^{\pi^*_{F,\gamma}}(s)-V_{F,\gamma}^{\pi^*_{\widehat{F},\gamma_{plan}}}(s)\right)$, which is always non-negative, we decompose it as follows:\hspace*{-0.1cm}
\begin{align}\label{eq:decompV}
\!\!\!\!\!\!\!\!\!\!\!\!\!\left(V_{F,\gamma}^{\pi^*_{F,\gamma}}(s)-V_{F,\gamma_{plan}}^{\pi^*_{F,\gamma}}(s)\right)+\left(V_{F,\gamma_{plan}}^{\pi^*_{F,\gamma}}(s)-V_{F,\gamma}^{\pi^*_{\widehat{F},\gamma_{plan}}}(s)\right)
\end{align}
The first term is the difference in the expected returns of $\pi^*_{F,\gamma}$ under two different discount factors, starting from $s$:
\begin{align}\label{eq:deltalambda}
\mathbb{E}\left[\sum_{t=0}^\infty (\gamma^t-\gamma^t_{plan})r_t|s_0 = s,\pi^*_{F,\gamma},F,M\right]\,.
\end{align}
Since $r_t\leq 1,~\forall t$, using the geometric series, Eq.~\ref{eq:deltalambda} is upper bounded by $\frac{1}{1-\gamma}-\frac{1}{1-\gamma_{plan}}=\frac{\gamma-\gamma_{plan}}{(1-\gamma_{plan})(1-\gamma)}$. 

The second term is upper bounded by $V_{F,\gamma_{plan}}^{\pi^*_{F,\gamma_{plan}}}(s)-V_{F,\gamma}^{\pi^*_{\widehat{F},\gamma_{plan}}}(s)$ since $\pi^*_{F,\gamma_{plan}}$ is an optimal policy of an environment equipped with $(F,\gamma_{plan})$. 
Furthermore, as $\gamma_{plan}\leq\gamma$ and $r_t\ge0$, we have
$V_{F,\gamma}^{\pi^*_{\widehat{F},\gamma_{plan}}}(s)\geq V_{F,\gamma_{plan}}^{\pi^*_{\widehat{F},\gamma_{plan}}}(s)$.
Therefore, the second term of Eq.~\ref{eq:decompV} is upper bounded by $V_{F,\gamma_{plan}}^{\pi^*_{F,\gamma_{plan}}}(s)-V_{F,\gamma_{plan}}^{\pi^*_{\widehat{F},\gamma_{plan}}}(s)$, which is the deviation of the value function under two different close policies. Since $F$ and $\widehat{F}$ are close, we expect that this deviation to be small. With one more decomposition step
\begin{align*}
&\!\!\!\!\! V_{F,\gamma_{plan}}^{\pi^*_{F,\gamma_{plan}}}(s)-V_{F,\gamma_{plan}}^{\pi^*_{\widehat{F},\gamma_{plan}}}(s)=
\left(V_{F,\gamma_{plan}}^{\pi^*_{F,\gamma_{plan}}}(s)-V_{\widehat{F},\gamma_{plan}}^{\pi^*_{F,\gamma_{plan}}}(s)\right)\nonumber\\ 
&\!\!\!\!\!\!\!\!\!\!\!\!\!\!\!+\left(V_{\widehat{F},\gamma_{plan}}^{\pi^*_{F,\gamma_{plan}}}(s)-V_{\widehat{F},\gamma_{plan}}^{\pi^*_{\widehat{F},\gamma_{plan}}}(s)\right)+
\left(V_{\widehat{F},\gamma_{plan}}^{\pi^*_{\widehat{F},\gamma_{plan}}}(s)-V_{F,\gamma_{plan}}^{\pi^*_{\widehat{F},\gamma_{plan}}}(s)\right)\,.
\end{align*}
Since the middle term in this equation is non-positive, we can ignore it for the purpose of upper-bounding the left-hand side. The upper bound is sum of the remaining two terms which is also upper bounded by 2 times of the maximum of them;
\begin{equation*}
2\max_{\pi\in\lbrace\pi^*_{F,\gamma_{plan}},\pi^*_{\widehat{F},\gamma_{plan}}\rbrace}\left|V_{\widehat{F},\gamma_{plan}}^{\pi}(s)-V_{F,\gamma_{plan}}^{\pi}(s)\right|\,,
\end{equation*}
which is the deviation in values of different domains. The value functions satisfy the Bellman equation for any $\pi$:
\begin{align}\label{eq:dynamic}
V_{{F},\gamma_{plan}}^{\pi}(s) = &\mathcal{R}(s,\pi(s))+\lambda F(s)\nonumber \\
+ \gamma_{plan}& \int_{s'\in\mathcal{S}}\!\!\!\!\!\!\!\!\mathcal{T}(s'|s,\pi(s))V_{{F},\gamma_{plan}}^{\pi}(s')ds\nonumber\\
V_{\widehat{F},\gamma_{plan}}^{\pi}(s) = &\mathcal{R}(s,\pi(s))+\lambda\widehat{F}(s) \\
+ \gamma_{plan} &\int_{s'\in\mathcal{S}}\!\!\!\!\!\!\!\!\mathcal{T}(s'|s,\pi(s))V_{\widehat{F},\gamma_{plan}}^{\pi}(s')ds
\end{align}
which can be solved using iterative updates of dynamic programing.
Let $V^\pi_i(s)$ and  $\widehat{V}^\pi_i(s) $ respectably denote the $i$'th iteration of the dynamic programmings corresponding to the first and second equalities in Eq.~\ref{eq:dynamic}. Therefore, for any state
\begin{align}
V^\pi_i(s) - &\widehat{V}^\pi_i(s) = 
\lambda'{F}(s) - \lambda'\widehat{F}(s) \nonumber\\
+& \gamma_{plan} \int_{s'\in\mathcal{S}}\!\!\!\!\!\!\!\!\mathcal{T}(s'|s,\pi(s))\left({V}_{i-1}(s')-\widehat{V}_{i-1}(s')\right)ds\nonumber\\
\leq &\lambda \sum_{i'=0}^{i}\left(\gamma_{plan}\mathcal{T}^{\pi}\right)^{i'}\left({F}-\widehat{{F}}\right)(s)\,,
\end{align}
where $(\mathcal{T}^{\pi})^{i}$ is a kernel and denotes the transition operator applied $i$ times to itself. The classification error $\left|{F}(s) - \widehat{F}(s)\right|$ is the zero-one loss of binary classifier, therefore, its expectation $\int_{s\in\mathcal{S}}\omega^{\pi^*_{\widehat{F},\gamma_{plan}}}(s)\left|{F}(s) - \widehat{F}(s)\right|ds$ is bounded by $3200\frac{\mathcal{VC}(\mathcal{F})+\log{\frac{1}{\delta}}}{N}$ with probability at least $1-\delta$~\cite{vapnik2013nature,hanneke2016optimal}. As long as the operator $(\mathcal{T}^{\pi})^{i}$ is a linear operator, 
\begin{align}\label{eq:vapnik}
\int_{s\in\mathcal{S}}\!\!\!\!\!\!\!\!\omega^{\pi^*_{\widehat{F},\gamma_{plan}}}(s)\left|V^\pi_i(s) - \widehat{V}^\pi_i(s)\right|ds\leq \lambda\frac{3200}{1-\gamma_{plan}}\frac{\mathcal{VC}(\mathcal{F})+\log{\frac{1}{\delta}}}{N}\,.
\end{align}
Therefore, $\mathcal{L}(F,\widehat{F},\gamma,\gamma_{plan})$ is bounded by $(1-\gamma)$ times of sum of Eq.~\ref{eq:vapnik} and $\frac{1-\gamma}{1-\gamma_{plan}}$, with probability at least $1-\delta$.
\end{proof}

Theorem~\ref{thm:classifier} holds for both finite and continuous state-action MDPs.
%
%
Over the course of our experiments, 
we discovered the following pattern: 
Intrinsic fear models are more effective 
when the \emph{fear radius} $k_r$
is large enough that the model 
can experience danger states at a safe distance 
and correct the policy, without experiencing many catastrophes. 
When the fear radius is too small, 
the danger probability is only nonzero 
at states from which catastrophes 
are inevitable anyway and 
intrinsic fear seems not to help. 
We also found that wider fear factors train more stably when 
phased in over the course of many episodes.
So, in all of our experiments 
we gradually phase in the \emph{fear factor} 
from $0$ to $\lambda$
reaching full strength at predetermined time step $k_{\lambda}$. 

\section{Environments}
\label{sec:failures}
We demonstrate our algorithms on the following environments: 
(i) \emph{Adventure Seeker}, a toy pathological environment  that we designed to demonstrate catastrophic forgetting; (ii) \emph{Cartpole}, a classic RL environment;
and (ii) the Atari games \emph{Seaquest}, \emph{Asteroids}, and \emph{Freeway}~\cite{bellemare2013arcade}.

\paragraph{Adventure Seeker}
We imagine a player placed on a hill, 
sloping upward to the right (Figure \ref{fig:adventure-seeker}).
At each turn, the player can move 
to the right (up the hill) or left (down the hill).
The environment adjusts the player's position accordingly, 
adding some random noise.
Between the left and right edges of the hill, 
the player gets more reward 
for spending time higher on the hill.
But if the player goes too far to the right, 
she will fall off, 
terminating the episode (catastrophe).
Formally, the state is single continuous variable 
$s \in [0,1.0]$, 
denoting the player's position.
The starting position for each episode 
is chosen uniformly at random in the interval $[.25, .75]$. 
The available actions consist only 
of $\{-1, +1\}$ (\emph{left} and \emph{right}).
Given an action $a_t$ in state $s_t$, $\mathcal{T}(s_{t+1}|s_t, a_t)$ 
the successor state is produced
$s_{t+1} \gets s_{t} + .01 \cdot a_t + \eta$
where $\eta \sim \mathcal{N}(0,.01^2)$.
The reward at each turn is $s_t$ 
(proportional to height).
The player falls off the hill, 
entering the catastrophic terminating state, 
whenever $s_{t+1} > 1.0$ or $s_{t+1} < 0.0$.

\begin{figure}[t]
\centering
\subfigure[Adventure Seeker]{
\includegraphics[width=.45\textwidth]{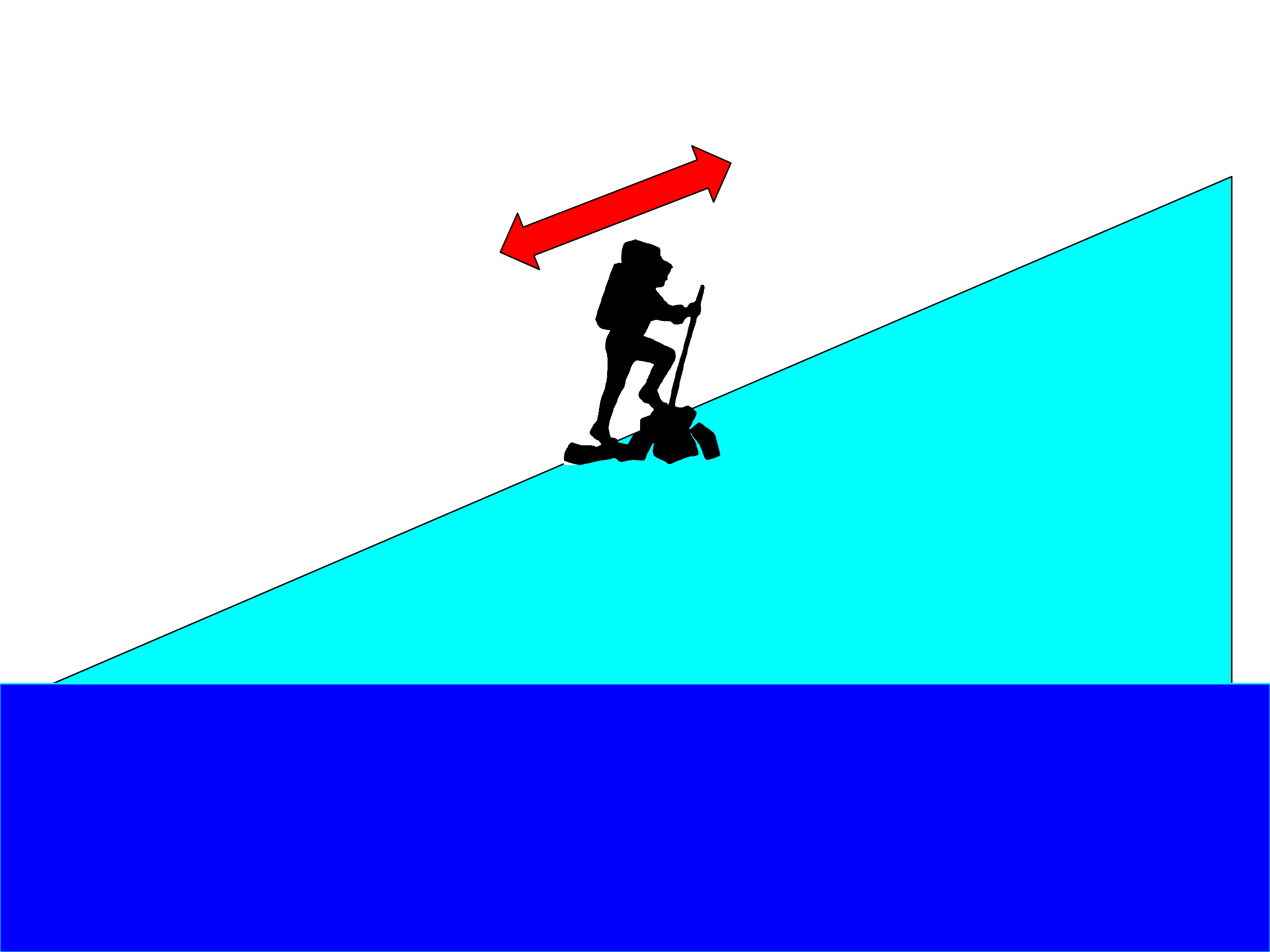}
\label{fig:adventure-seeker}
}
\subfigure[Cart-Pole]{
\includegraphics[width=.45\textwidth]{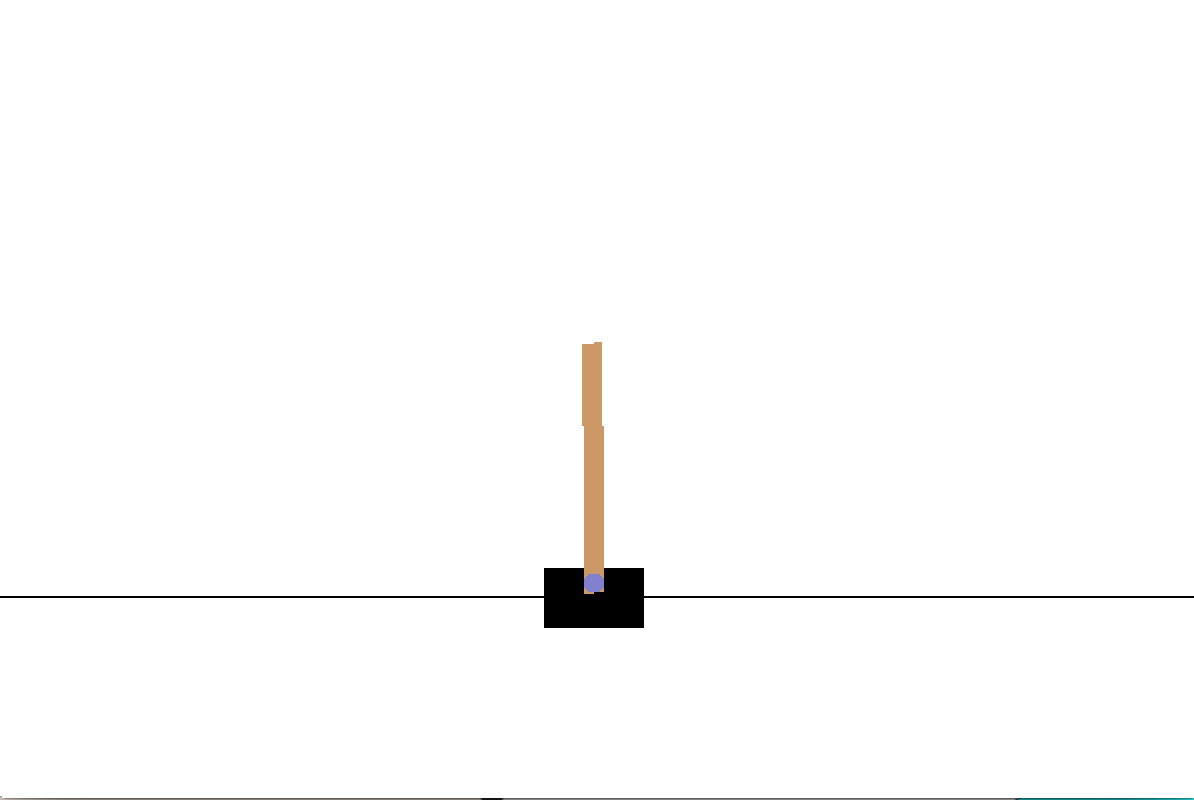}
\label{fig:cart-pole}
}

\subfigure[Seaquest]{
\includegraphics[width=.3\textwidth]{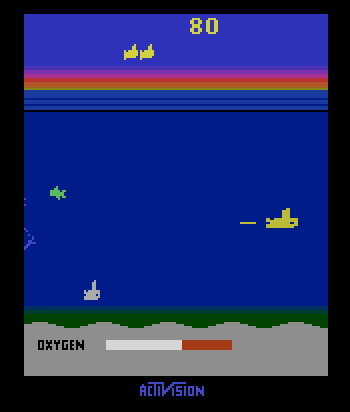}
\label{fig:seaquest}
}
\subfigure[Asteroids]{
\includegraphics[width=.3\textwidth]{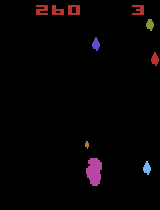}
\label{fig:asteroids}
}
\subfigure[Freeway]{
\includegraphics[width=.3\textwidth]{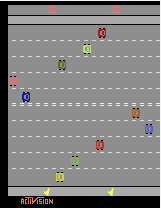}
\label{fig:freeway}
}

\caption{In experiments, we consider two toy environments (a,b) and the Atari games Seaquest (c), Asteroids (d), and Freeway (e)}
\label{fig:domains}
\end{figure}


This game should be easy to solve.
There exists a threshold above which 
the agent should always choose to go left 
and below which it should always go right. 
And yet a DQN agent
will periodically die. 
Initially, the DQN quickly learns a good policy and avoids the catastrophe,
but over the course of continued training,
the agent, owing to the shape of the reward function,
collapses to a policy which always moves right, 
regardless of the state.
We might critically ask in what real-world scenario, 
we could depend upon a system 
that cannot solve \emph{Adventure Seeker}.


\paragraph{Cart-Pole}
In this classic RL environment,
an agent balances a pole atop a cart (Figure \ref{fig:cart-pole}).
Qualitatively, the game exhibits four distinct catastrophe modes.
The pole could fall down to the right or fall down to the left.
Additionally, the cart could run off the right boundary of the screen or run off the left. 
Formally, at each time, the agent observes a four-dimensional  state vector $(x,v,\theta, \omega)$ consisting respectively of the cart position, cart velocity, pole angle, and the pole's angular velocity.
At each time step, the agent chooses an action, 
applying a force of either $-1$ or $+1$.
For every time step that the pole remains upright and the cart remains on the screen, the agent receives a reward of $1$. 
If the pole falls, the episode terminates, giving a return of $0$ from the penultimate state. 
In experiments, we use the implementation \emph{CartPole-v0}
contained in the openAI gym~\cite{brockman2016gym}.
Like Adventure Seeker, this problem admits an analytic solution. A perfect policy should never drop the pole. But, as with Adventure Seeker, a DQN converges to a constant rate of  catastrophes per turn. 

\paragraph{Atari games} 
In addition to these pathological cases, 
we address Freeway, Asteroids, and Seaquest, 
games from the Atari Learning Environment.
In Freeway, the agent controls a chicken
with a goal of crossing the road 
while dodging traffic. 
The chicken loses a life and starts from the original location 
if hit by a car. 
Points are only rewarded for successfully crossing the road.
In Asteroids, the agent pilots a ship 
and gains points from shooting the asteroids.
She must avoid colliding with asteroids which cost it lives.
In Seaquest, a player swims under water. 
Periodically, as the oxygen gets low, 
she must rise to the surface for oxygen.
Additionally, fishes swim across the screen. 
The player gains points each time she shoots a fish. 
Colliding with a fish or running out of oxygen result in death. 
In all three games, the agent has 3 lives, 
and the final death is a terminal state. 
We label each loss of a life as a catastrophe state.

\section{Experiments}
\label{sec:experiments}
First, on the toy examples,
We evaluate standard DQNs
and \emph{intrinsic fear} DQNs
using multilayer perceptrons (MLPs) 
with a single hidden layer and $128$ hidden nodes. 
We train all MLPs by stochastic gradient descent 
using the Adam optimizer \cite{kingma2014adam}. 

In \emph{Adventure Seeker},
an agent can escape from danger 
with only a few time steps of notice, 
so we set the fear radius $k_r$ to $5$.  
We phase in the fear factor quickly, 
reaching full strength in just $1000$ steps. 
On this problem we set the fear factor 
$\lambda$ to $40$.
For \emph{Cart-Pole}, we set a 
wider fear radius of $k_r=20$.
We initially tried training this model
with a short fear radius but made the following observation:
One some runs, IF-DQN would surviving for millions of experiences, while on other runs, it might experience many catastrophes. 
Manually examining fear model output 
on successful vs unsuccessful runs,
we noticed that on the bad runs,
the fear model outputs non-zero probability of danger for precisely the $5$ moves before a catastrophe. In Cart-Pole, by that time,
it is too to correct course. 
On the more successful runs, 
the fear model often outputs predictions in the range $.1-.5$. 
We suspect that the gradation between mildly dangerous states and those with certain danger provides a richer reward signal to the DQN. 

On both the Adventure Seeker and Cart-Pole environments, DQNs augmented by intrinsic fear far outperform their otherwise identical counterparts. 
We also compared IF to some traditional approaches for mitigating catastrophic forgetting.
For example, we tried a memory-based method
in which we preferentially sample the catastrophic states for updating the model,
but they did not improve over the DQN.
It seems that the notion of a danger zone is necessary here. 

For Seaquest, Asteroids, and Freeway, we use a fear radius of $5$ 
and a fear factor of $.5$. 
For all Atari games, 
the IF models outperform their DQN counterparts. 
Interestingly while for all games, the IF models achieve higher reward, 
on Seaquest, IF-DQNs  
have similar catastrophe rates (Figure \ref{fig:fear-results}).
Perhaps the IF-DQN enters a region of policy space with a strong incentives to exchange catastrophes for higher reward.
This result suggests an interplay between the various reward signals that warrants further exploration. 
For Asteroids and Freeway, the improvements are more dramatic. Over just a few thousand episodes of Freeway, a randomly exploring DQN achieves zero reward. However, the reward shaping 
of intrinsic fear leads to rapid improvement.

\begin{figure}[t]
\centering
\subfigure[Seaquest]{
\includegraphics[width=.30\linewidth,height=2.cm]{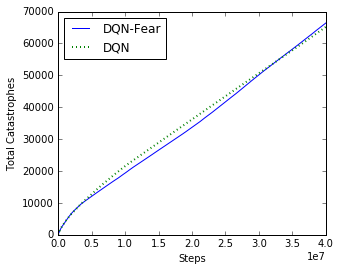}
\label{fig:seaquest-catastrophe-lives}
}
\subfigure[Asteroids]{
\includegraphics[width=.30\linewidth,height=2.cm]{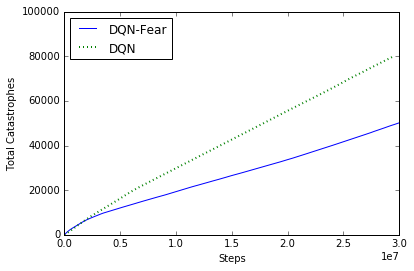}
\label{fig:asteroids-catastrophe-lives}
}
\subfigure[Freeway]{
\includegraphics[width=.30\linewidth,height=2.cm]{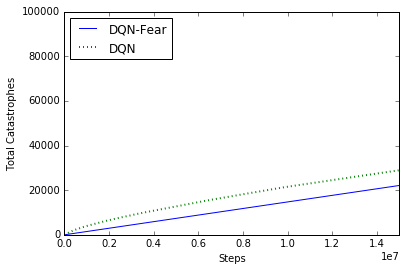}
\label{fig:freeway-catastrophe-lives}
}
\subfigure[Seaquest]{
\includegraphics[width=.30\linewidth,height=2.cm]{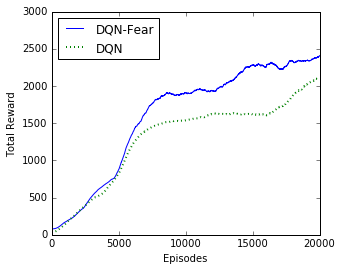}
\label{fig:seaquest-reward}
}
\subfigure[Asteroids]{
\includegraphics[width=.30\linewidth,height=2.cm]{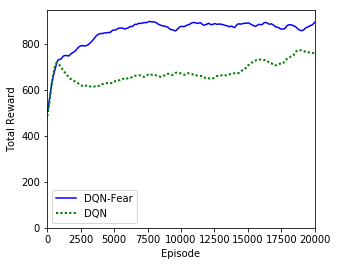}
\label{fig:asteroids-reward}
}
\subfigure[Freeway]{
\includegraphics[width=.30\linewidth,height=2.cm]{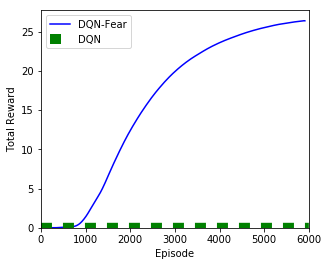}
\label{fig:freeway-reward}
}
\caption{Catastrophes (first row) and reward/episode (second row) for DQNs and \emph{Intrinsic Fear}. 
On Adventure Seeker, all Intrinsic Fear models 
cease to ``die'' within $14$ runs, 
giving unbounded (unplottable) reward thereafter. 
On Seaquest, the IF model achieves a similar catastrophe rate 
but significantly higher total reward. 
On Asteroids, the IF model outperforms DQN. 
For Freeway, a randomly exploring DQN (under our time limit) never gets reward but IF model learns successfully.}
\label{fig:fear-results}
\end{figure}

\section{Related work}
\label{sec:related}
The paper studies safety in RL, intrinsically motivated RL, and the stability of Q-learning with function approximation under distributional shift. Our work also has some connection to reward shaping. We attempt to highlight the most relevant papers here.  Several papers address safety in RL.
\namecite{garcia2015comprehensive} provide a thorough review on the topic, 
identifying two main classes of methods:  those that perturb the objective function 
and those that use external knowledge 
to improve the safety of exploration.

While a typical reinforcement learner 
optimizes expected return, 
some papers suggest 
that a safely acting agent 
should also minimize risk. 
\namecite{hans2008safe} defines a \emph{fatality} 
as any return below some threshold $\tau$. 
They propose a solution comprised of 
a \emph{safety function}, 
which identifies unsafe states,
and a \emph{backup model}, 
which navigates away from those states.
Their work, which only addresses the tabular setting,
suggests that an agent 
should minimize the probability of fatality 
instead of maximizing the expected return. 
\namecite{heger1994consideration} 
suggests an alternative 
Q-learning objective concerned with the minimum (vs. expected) return. 
Other papers suggest modifying the objective to penalize policies with high-variance returns \cite{garcia2015comprehensive,chow15cvar}. 
Maximizing expected returns 
while minimizing their variance 
is a classic problem in finance,
where a common objective is the ratio of expected return to its standard deviation \cite{sharpe1966mutual}. Moreover, \namecite{azizzadenesheli2018efficient} suggests to learn the variance over the returns in order to make safe decisions at each decision step.
\namecite{moldovan2012safe} give a definition 
of safety based on ergodicity. 
They consider a fatality to be a state 
from which one cannot return to the start state.
\namecite{shalev2016safe} theoretically analyzes how strong a penalty should be
to discourage accidents.
They also consider hard constraints to ensure safety.
None of the above works address the case 
where distributional shift 
dooms an agent to perpetually revisit known catastrophic failure modes.
Other papers incorporate external knowledge into the exploration process.
Typically, this requires access to an oracle 
or extensive prior knowledge of the environment.
In the extreme case, some papers suggest 
confining the policy search to a known subset of \emph{safe} policies.
For reasonably complex environments or classes of policies, this seems infeasible.

The potential oscillatory or divergent behavior of Q-learners with function approximation has been previously identified \cite{boyan1995generalization,baird1995residual,gordon1996chattering}. 
Outside of RL, the problem of covariate shift has been extensively studied 
\cite{sugiyama2012machine}.
\namecite{murata2005memory} addresses the problem of catastrophic forgetting owing to distributional shift in RL 
with function approximation, 
proposing a memory-based solution.
Many papers address intrinsic rewards,
which are internally assigned, 
vs the standard (extrinsic) reward. 
Typically, intrinsic rewards 
are used to encourage exploration 
\cite{jurgen1991possibility,bellemare2016unifying} 
and to acquire a modular set of skills \cite{chentanez2004intrinsically}.
Some papers refer to the intrinsic reward for discovery as \emph{curiosity}.
Like classic work on intrinsic motivation,
our methods perturb the reward function.
But instead of assigning bonuses to encourage discovery of novel transitions, 
we assign penalties to discourage catastrophic transitions.
\paragraph{Key differences} 
In this paper, we undertake a novel treatment 
of safe reinforcement learning,
While the literature offers several notions of safety in reinforcement learning, 
we see the following problem: 
Existing safety research that perturbs 
the reward function requires little foreknowledge,
but fundamentally changes the objective globally. 
On the other hand, processes relying on expert knowledge 
may presume an unreasonable level of foreknowledge.
%
Moreover, little of the prior work on safe reinforcement learning, to the best of our knowledge, specifically addresses the problem of catastrophic forgetting. 
This paper proposes a new class of algorithms for avoiding catastrophic states and a theoretical analysis supporting its robustness. 

\section{Conclusions}
\label{sec:discussions}
Our experiments demonstrate that DQNs
are susceptible to periodically repeating mistakes, however bad, raising questions about their real-world utility
when harm can come of actions. 
While it is easy to visualize these problems on toy examples, 
similar dynamics are embedded in more complex domains.
Consider a domestic robot acting as a barber. 
The robot might receive positive feedback for giving a closer shave. 
This reward encourages closer contact at a steeper angle.
Of course, the shape of this reward function belies the catastrophe lurking just past the optimal shave. 
Similar dynamics might be imagines 
in a vehicle that is rewarded 
for traveling faster 
but could risk an accident 
with excessive speed. 
Our results with the intrinsic fear 
model suggest 
that with only a small amount of prior knowledge (the ability to recognize catastrophe states after the fact), 
we can simultaneously accelerate learning and avoid catastrophic states. 
This work is a step towards combating  
DRL's tendency to revisit catastrophic states due to catastrophic forgetting.

 


\bibliographystyle{named}
\bibliography{intrinsic-fear}

\newpage
\onecolumn

\appendix
\section*{An extension to the Theorem \ref{thm:classifier}}
In practice, we gradually learn and improve $\widehat{F}$ where the difference between learned $\widehat{F}$ after two consecrative updates, $\widehat{F}_{t}$ and $\widehat{F}_{t+1}$, consequently, $\omega^{\pi^*_{\widehat{F}_t,\gamma_{plan}}}$ and $\omega^{\pi^*_{\widehat{F}_{t+1},\gamma_{plan}}}$ decrease. While $\widehat{F}_{t+1}$ is learned through using the samples drawn from $\omega^{\pi^*_{\widehat{F}_t,\gamma_{plan}}}$, with high probability
\begin{align*}
\int_{s\in\mathcal{S}}\omega^{\pi^*_{\widehat{F}_t,\gamma_{plan}}}(s)\left|{F}(s) - \widehat{F}_{t+1}(s)\right|ds\leq
3200\frac{\mathcal{VC}(\mathcal{F})+\log{\frac{1}{\delta}}}{N}
\end{align*}

But in the final bound in Theorem \ref{thm:classifier}, we interested in 
$\int_{s\in\mathcal{S}}\omega^{\pi^*_{\widehat{F}_{t+1},\gamma_{plan}}}(s)\left|{F}(s) - \widehat{F}_{t+1}(s)\right|ds$. Via decomposing in into two terms 
\begin{align*}
\int_{s\in\mathcal{S}}\omega^{\pi^*_{\widehat{F}_t,\gamma_{plan}}}(s)\left|{F}(s) - \widehat{F}_{t+1}(s)\right|ds+
\int_{s\in\mathcal{S}}|\omega^{\pi^*_{\widehat{F}_{t+1},\gamma_{plan}}}(s)-\omega^{\pi^*_{\widehat{F}_t,\gamma_{plan}}}(s)|ds
\end{align*}
Therefore, an extra term of
$\lambda\frac{1}{1-\gamma_{plan}}\int_{s\in\mathcal{S}}|\omega^{\pi^*_{\widehat{F}_{t+1},\gamma_{plan}}}(s)-\omega^{\pi^*_{\widehat{F}_t,\gamma_{plan}}}(s)|ds$ appears in the final bound of Theorem \ref{thm:classifier}. 

Regarding the choice of $\gamma_{plan}$, if $\lambda\frac{\mathcal{VC}(\mathcal{F})+\log{\frac{1}{\delta}}}{N}$ is less than one, then the best choice of $\gamma_{plan}$ is $\gamma$. Other wise, if $\frac{\mathcal{VC}(\mathcal{F})+\log{\frac{1}{\delta}}}{N}$ is equal to exact error in the model estimation, and is greater than $1$, then the best $\gamma_{plan}$ is 0. Since, $\frac{\mathcal{VC}(\mathcal{F})+\log{\frac{1}{\delta}}}{N}$ is an upper bound, not an exact error, on the model estimation, the choice of zero for $\gamma_{plan}$ is not recommended, and a choice of $\gamma_{plan}\leq \gamma$ is preferred.


\end{document}